\documentclass[conference]{IEEEtran}
\IEEEoverridecommandlockouts
\pdfoutput=1

\usepackage{color}
\usepackage{bm}
\usepackage{hyperref}
\usepackage{textcomp}
\usepackage{color}
\usepackage{amsthm}
\usepackage{amsfonts}
\usepackage{times,amsmath}
\usepackage{amssymb}
\usepackage{graphicx}
\usepackage{xcolor}
\usepackage{subcaption}
\usepackage{cite}
\usepackage{url}
\usepackage[capitalise]{cleveref}
\usepackage{enumitem}
\usepackage{tikz, pgfplots}
\usepackage[ruled,linesnumbered]{algorithm2e}
\usepackage{moresize} 

\input{mysymbol.sty}

\usetikzlibrary{shapes,arrows}
\pgfplotsset{compat=1.17}
\pgfplotstableset{col sep=semicolon}
\usepgfplotslibrary{groupplots}
\usepgfplotslibrary{fillbetween}

\tikzset{every mark/.append style={scale=1.6, solid}, font=\small}
\pgfplotsset{
    width=1\textwidth,
    legend style={
        font=\scriptsize ,  
        inner xsep=1pt,
        inner ysep=1pt,
        nodes={inner sep=1pt}},
    legend cell align=left,
    every axis/.append style={line width=.5pt},
 	every axis plot/.append style={line width=1.5pt},
 	every axis y label/.append style={yshift=-4pt}
}

\begin{document}

\title{Non-negative Weighted DAG Structure Learning
\thanks{This work was supported in part by the Spanish AEI Grants PID2022-136887NB-I00, TED2021-130347B-I00, PID2023-149457OB-I00, and the Community of Madrid (Madrid ELLIS Unit).}
}

\author{
\IEEEauthorblockN{
Samuel Rey\IEEEauthorrefmark{1},
Seyed Saman Saboksayr\IEEEauthorrefmark{2},
Gonzalo Mateos\IEEEauthorrefmark{2},
} %
\IEEEauthorblockA{
\IEEEauthorrefmark{1}Dept. of Signal Theory and Communications, Rey Juan Carlos University, Madrid, Spain } %
\IEEEauthorblockA{
\IEEEauthorrefmark{2}Dept. of Electrical and Computer Eng., University of Rochester, Rochester, NY, United States } %
}

\maketitle
\begin{abstract}
We address the problem of learning the topology of directed acyclic graphs (DAGs) from nodal observations, which adhere to a linear structural equation model. Recent advances framed the combinatorial DAG structure learning task as a continuous optimization problem, yet existing methods must contend with the complexities of non-convex optimization. To overcome this limitation, we assume that the latent DAG contains only non-negative edge weights. Leveraging this additional structure, we argue that cycles can be effectively characterized (and prevented) using a convex acyclicity function based on the log-determinant of the adjacency matrix. This convexity allows us to relax the task of learning the non-negative weighted DAG as an abstract convex optimization problem. We propose a DAG recovery algorithm based on the method of multipliers, that is guaranteed to return a global minimizer. Furthermore, we prove that in the infinite sample size regime, the convexity of our approach ensures the recovery of the true DAG structure. We empirically validate the performance of our algorithm in several reproducible synthetic-data test cases, showing that it outperforms state-of-the-art alternatives.
\end{abstract}

\section{Introduction}\label{s:Introduction}
Directed acyclic graphs (DAGs) are crucial tools for modeling complex systems where directionality plays a key role~\cite{marques2020digraphs}, and they are widely recognized for their ability to represent causal relationships~\cite{peters2017elements, seifert2023causal}.
Consequently, DAGs and associated Bayesian networks have become increasingly common tools in biology~\cite{sachs2005causal,lucas2004bayesian}, genetics~\cite{zhang2013integrated}, machine learning~\cite{koller2009probabilistic,yu2019dag,rey2024convolutional}, signal processing~\cite{seifert2023causal,misiakos2024learning}, and causal inference~\cite{spirtes2001causation,yao2021survey}.
Despite their widespread adoption, often the DAG structure is not known in advance and must be inferred from data.

Learning a graph from nodal observations is a prominent problem rooted in the relation between the properties of the observed data and the graph topology~\cite{mateos2019connecting}.
Noteworthy approaches for undirected graphs include Gaussian graphical models~\cite{friedman2008sparse,egilmez2017graph,rey2023enhanced}, smoothness 
models~\cite{kalofolias2016learn,dong2016learning,saboksayr2021accelerated}, or graph stationary models~\cite{segarra2017network,shafipour2020online,roddenberry2021network,buciulea2022learning,navarro2024joint}, among others.
When the graph of interest is a DAG, structural equation models (SEMs) are often the method of choice~\cite{zheng2018dags,wei2020dags,bello2022dagma,saboksayr2023colide}.
Accounting for the acyclicity of the graph renders the DAG structure learning a challenging combinatorial, in fact NP-hard, endeavor~\cite{maxwell1997efficient,chickering2004large}. 

Recent works managed to circumvent the combinatorial nature of DAG structure learning by introducing a continuous relaxation that allows for efficient exploration of the DAG space.
A breakthrough in~\cite{zheng2018dags} advocated a continuous non-convex acyclicity constraint based on the matrix exponential.
This inspired further developments, including acyclicity functions based on powers of the adjacency matrix~\cite{wei2020dags,pamfil2020dynotears} and the log determinant~\cite{ng2020role,bello2022dagma,saboksayr2023colide}.
Despite significant progress, learning the DAG structure remains a challenging task involving a non-convex optimization problem.
Consequently, current methods rely on heuristics~\cite{bello2022dagma} and are content with estimates corresponding to local minima; see also~\cite{deng2023global}.

\vspace{2mm}
\noindent
\textbf{Contributions.}
To circumvent these non-convexity issues, 
we focus on the class of DAGs with non-negative edge weights and propose a \emph{convex acyclicity function} that enables recovering the global minimizer. To the best of our knowledge, this is the first work proposing a convex relaxation for DAG estimation.
We contribute the following technical innovations:
\begin{itemize}
    \item By leveraging the non-negativity of the DAG edge weights, we propose a convex log-determinant function to characterize the acyclicity of the graph.
    \item We cast the DAG learning task as an abstract convex optimization problem and propose an iterative algorithm based on the method of multipliers.
    This approach warrants the recovery of the global minimum.
    \item We prove that the proposed method recovers the true DAG structure when infinite observations are available. 
\end{itemize}

\section{Fundamentals of DAG structure learning}

Let $\ccalD = (\ccalV, \ccalE)$ denote a DAG, where $\ccalV$ is a set of $d$ nodes, and $\ccalE \subseteq \ccalV \times \ccalV$ is a set of directed edges.
An edge $(i, j) \in \ccalE$ exists if and only if there is a directed link from node $i$ to node $j$.
The connectivity of the DAG is captured by the weighted adjacency matrix $\bbW \in \reals^{d \times d}$, where $W_{ij} \neq 0$ if and only if $(i, j) \in \ccalE$.
Then, a graph signal is defined on the nodes of the DAG and is represented as a vector $\bbx \in \reals^d$, with $x_i$ denoting the signal value at node $i$.

The task of DAG structure learning involves inferring the topology of a DAG from a set of nodal observations. Collecting the $n$ observed signals in the matrix $\bbX := [\bbx_1,\ldots, \bbx_n] \in \reals^{d \times n}$, we suppose $\bbX$ adheres to a linear SEM given by 
\begin{equation}
	\bbX = \bbW^\top \bbX + \bbZ,
\end{equation}
where $\bbZ \!\in\! \reals^{d \times n}$ collects zero-mean exogenous noises whose columns are i.i.d. random vectors with covariance matrix $\bbSigma_\bbz = \sigma^2\bbI$.
Mutual independence of the noise variables is crucial~\cite[pp. 83-84]{peters2017elements}. 

With the previous definitions in place, the DAG structure encoded in $\bbW$ can be inferred from $\bbX$ by solving the optimization problem
\begin{equation}\label{eq:dag_learning} 
    \min_{\bbW}  F \left( \bbW, \bbX \right) \quad \mathrm{s. to } \quad \bbW \in \mbD,
\end{equation}
%
where $F(\bbW, \bbX)$ denotes a data-dependent score function that captures the relation between $\bbW$ and the signals, and $\mbD$ denotes the set of adjacency matrices corresponding to a DAG.

The optimization problem in \eqref{eq:dag_learning} is challenging to solve due to the non-convex and combinatorial nature of the constraint $\bbW \in \mbD$. Recent advances advocate recovery of DAG structure by replacing this constraint with an \emph{acyclicity condition} of the form $h(\bbW) = 0$, where $h: \reals^{d \times d} \mapsto \reals$ is a differentiable function whose zero level set corresponds to $\mbD$.
This approach was pioneered in~\cite{zheng2018dags} via the acyclicity function
\begin{equation}\label{eq:notears}
	h_{\mathrm{notears}}(\bbW) = \tr \left( e^{\bbW \circ \bbW} \right) - d,
\end{equation}
where $\circ$ denotes the Hadamard (entry-wise) product.
More recently,~\cite{bello2022dagma} proposed an alternative acyclicity characterization 
\begin{equation}\label{eq:dagma}
		h_{\mathrm{dagma}}^s(\bbW) = d \log (s) - \log\det \left( s\bbI - \bbW \circ \bbW \right),
\end{equation}
with $s \in \reals_+$.
The log-determinant function alleviates numerical issues and has been shown to outperform prior relaxations.

Replacing the combinatorial constraint $\bbW \in \mbD$ with the acyclicity condition $h(\bbW) = 0$ constitutes a significant advancement, enabling the use of standard continuous optimization methods to learn the DAG structure.
However, the presence of the term $\bbW \circ \bbW$ renders these functions non-convex, still posing important challenges to recovering the true DAG.
To overcome this limitation, we henceforth 
assume $\bbW$ \emph{has non-negative weights} and propose: (i) a convex alternative to \eqref{eq:dagma}; and (ii) a method to obtain the global minimizer of an optimization problem equivalent to \eqref{eq:dag_learning}.   


\section{Non-negative DAG structure learning}
We tackle the problem of learning the DAG structure by assuming that the entries of $\bbW$ are non-negative. The restriction is still relevant to binary DAGs and other pragmatic settings; see e.g.,~\cite{seifert2023causal}. This additional structure is crucial for simplifying the optimization problem in \eqref{eq:dag_learning}, allowing us to characterize the acyclicity of the graph using a convex function. 

When $\bbX$ adheres to a linear SEM, a common choice for the score function is $F(\bbW,\bbX) = \frac{1}{2n} \| \bbX - \bbW^\top \bbX \|_F^2 + \alpha \| \bbW \|_1$, which balances a data-fidelity term with $\ell_1$ norm regularization to encourage sparse solutions.
This trade-off is controlled by the tunable weight $\alpha \in \reals_+$.
Using this convex score function, a continuous acyclicity constraint, and the entrywise non-negativity of $\bbW$, the DAG structure learning problem can be alternatively formulated as
\begin{alignat}{3}\label{eq:nonneg_dag_learning}
	\!\!&\! \hbW = && \arg\min_{\bbW} \
	&& \left\{\frac{1}{2n}\| \bbX - \bbW^\top\bbX \|_F^2 + \alpha \sum_{i,j=1}^d W_{ij} \right\}  \nonumber \\
	\!\!&\! && \mathrm{s. to} && 
	\bbW \geq 0, \;\; h(\bbW) = 0,
\end{alignat}
where $h(\bbW)$ denotes an acyclicity function of interest, and the $\ell_1$ norm is replaced by $\sum_{i,j=1}^d W_{ij}$ due to the non-negativity of $\bbW$. The acyclicity constraint renders the optimization problem in \eqref{eq:nonneg_dag_learning} non-convex, an issue that is dealt with next.


\subsection{Convex acyclicity functions via non-negative matrices}
Relying on a smooth acyclicity constraint to ensure that $\bbW$ is cycle-free is central to modern DAG learning methods.
As discussed in~\cite{zheng2018dags} an effective acyclicity function $h(\bbW)$ should be smooth, have an easy to compute gradient, and satisfy $h(\bbW) = 0$ if and only if $\bbW\in\mathbb{D}$. Harnessing the non-negativity of $\bbW$ and inspired by \eqref{eq:dagma}, we introduce a \emph{convex} acyclicity function that meets these criteria.

\begin{proposition}\label{prop:cvx_logdet}
	For any matrix $\bbW \in \reals_+^{d \times d}$ whose spectral radius is bounded by $\rho(\bbW) < s$ with $s \in \reals_+$, define
	\begin{equation}\label{eq:cvx_logdet}
		h_{ldet}(\bbW) := d\log(s) - \log\det(s\bbI - \bbW),
	\end{equation}
	with gradient $\nabla h_{ldet} (\bbW) = \left(s\bbI - \bbW \right)^{-\top}$.
	Then, $h_{ldet}(\bbW) \geq 0$ 
for every $\bbW$ such that $\rho(\bbW) < s$, 
 and $h_{ldet}(\bbW) = 0$ if and only if $\bbW\in\mathbb{D}$.
\end{proposition}
\begin{proof}
We start by rewriting $h_{ldet}$ in the equivalent form
\begin{align}
h_{ldet}(\bbW) &= d\log(s) -\log(s^d) - \log\det(\bbI - s^{-1}\bbW) \nonumber \\
&= - \log\det(\bbI - s^{-1}\bbW).  \nonumber
\end{align}
From the bounded spectral radius $\rho(\bbW) < s$, it follows that $\log\det(\bbI - s^{-1}\bbW)$ is well defined for every $\bbW$, so for the rest of the proof we set $s = 1$ without loss of generality.
	
First, we show that $h_{ldet}(\bbW) \geq 0$ for every non-negative $\bbW$, and then establish $h_{ldet}(\bbW) = 0\Leftrightarrow\bbW\in\mathbb{D}$. To that end, we start by showing that $\log\det(\bbI - \bbW) \leq 0$. With $\lambda_i(\bbB)$ denoting the $i$-th eigenvalue of some matrix $\bbB$, we have
\begin{align}
		\log\det &(\bbI \!-\! \bbW) \! = \! \sum_{i=1}^d \! \log \left( \lambda_i(\bbI \!-\!  \bbW) \right) \! = \! d\sum_{i=1}^d \! \frac{\log \left( \lambda_i(\bbI - \bbW) \right)}{d} \nonumber \\
		&\leq d \log \left( \sum_{i=1}^d \frac{ \lambda_i(\bbI - \bbW) }{d} \right) = d \log \left( \frac{\tr (\bbI - \bbW)}{d} \right), \nonumber
	\end{align}
where the inequality follows from Jensen's inequality for concave functions.
	
Next, we leverage that $\bbW \geq 0$, and hence, $\tr(\bbW) \geq 0$, to obtain the bound $\tr(\bbI - \bbW) = \tr(\bbI) - \tr(\bbW) \leq \tr(\bbI)=d$. Combining this with the monotonicity of the logarithm renders
	\begin{equation}
		\log\det(\bbI - \bbW) \leq d \log \left( \frac{\tr (\bbI - \bbW)}{d} \right) \leq d \log\left( \frac{\tr(\bbI)}{d} \right) = 0. \nonumber
	\end{equation}
Therefore, $h_{ldet}(\bbW) = - \log\det(\bbI - \bbW) \geq 0$, as intended.
	
Finally, note that $h_{ldet}(\bbW) = - \log\det(\bbI - \bbW) = 0$ if and only if all the eigenvalues of $\bbW$ are zero, which means that $\bbW$ is a nilpotent matrix, equivalent to $\bbW$ being a DAG~\cite{bello2022dagma}.
\end{proof}


Ensuring the acyclicity of $\bbW$ using a convex function such as $h_{ldet}$ offers significant advantages. First, the convexity of $h_{ldet}$ renders \eqref{eq:nonneg_dag_learning} an \emph{abstract convex optimization problem~\cite{boyd2004convex}}, enabling us to reliably recover the global minimum. To see this, note that under the conditions of \cref{prop:cvx_logdet}, the feasible set defined by $h_{ldet}(\bbW) = 0$ is a convex set and, in fact, is equivalent to the feasible set of the convex constraint $h_{ldet}(\bbW) \leq 0$. In turn, recovering the global minimum provides new opportunities to characterize the estimate $\hbW$, a promising research direction that we start pursuing next, leaving a more in-depth analysis as future work. 
Moreover, unlike the acyclicity functions in \eqref{eq:notears} and \eqref{eq:dagma}, $h_{ldet}$ does not have stationary points at DAGs, meaning $\nabla h_{ldet} (\bbW^\star) \neq 0$ for $\bbW^\star \in \mbD$.
This avoids algorithmic issues highlighted in~\cite{wei2020dags}.

Similar to \cref{prop:cvx_logdet}, when $\bbW$ is non-negative, a convex alternative to \eqref{eq:notears} is given by
\begin{equation}\label{eq:cvx_matexp}
    h_{\mathrm{mexp}}(\bbW) = \tr \left( e^{\bbW} \right) - d,
\end{equation}
where $h_{\mathrm{mexp}}(\bbW) = 0$ if and only if $\bbW \in \mbD$~\cite{wei2020dags}. While \eqref{eq:cvx_matexp} is also a convex function, acyclicity constraints based on the log-determinant have demonstrated superior performance~\cite{bello2022dagma,saboksayr2023colide}.
Indeed, we further examine how different acyclicity functions impact the DAG recovery in \cref{sec:exps}.

\subsection{DAG structure learning via method of multipliers}\label{sec:algorithm}
Considering a convex acyclicity function $h(\bbW)$, we solve the optimization problem in \eqref{eq:nonneg_dag_learning} using the method of multipliers~\cite[Ch. 4.2]{bertsekas1997dynamic}, an iterative approach based on the augmented Lagrangian tailored to constrained optimization problems.

Let the augmented Lagrangian of \eqref{eq:nonneg_dag_learning} be given by
\begin{align}
	L_c(\bbW, \lambda) = \frac{1}{2n}\| \bbX - \bbW^\top\bbX \|_F^2 &+ \alpha \sum_{i,j=1}^d W_{ij} + \lambda h(\bbW) \nonumber \\
    &+ \frac{c}{2} h(\bbW)^2,
\end{align}
where $\lambda \in \reals_+$ is the Lagrange multiplier, and $c \in \reals_+$ is a penalty parameter.
Note that the term $h(\bbW)^2$ is convex since it is a composition of a convex function and a convex and non-decreasing function~\cite{boyd2004convex}, and hence, the augmented Lagrangian is convex.
Note that no term is related to the constraint $\bbW \geq 0$ since it can be enforced through a simple projection. 
Then, at each iteration $k=1,\ldots,\kappa_{max}$, we perform the following sequence of steps.

\vspace{2mm}
\noindent
\textbf{Step 1.}
We update $\bbW^{(k+1)}$ by minimizing
\begin{equation}\label{eq:step1}
	\bbW^{(k+1)} = \arg\min_{\bbW \geq 0} L_{c^{(k)}} (\bbW, \lambda^{(k)}).
\end{equation}
Thanks to the convexity of $L_{c^{(k)}}$, we can recover the global minimum $\bbW^{(k+1)}$ with standard convex optimization methods such as projected gradient descent.

\vspace{2mm}
\noindent
\textbf{Step 2.}
The update of the Langrange multiplier $\lambda^{(k+1)}$ depends on the degree of constraint violation, given by
\begin{equation}\label{eq:step2}
	\lambda^{(k+1)} = \lambda^{(k)} + c^{(k)}h(\bbW^{(k+1)}).
\end{equation}
This update can also be interpreted as a gradient ascent step since the constraint violation corresponds to the gradient of $L_{c^{(k)}}(\bbW^{(k+1)},\lambda)$ with respect to $\lambda$.

\vspace{2mm}
\noindent
\textbf{Step 3.}
The penalty parameter $c^{(k)}$ needs to be progressively increased so the constraint $h(\bbW) = 0$ is satisfied as $\kappa_{max}\to \infty$.
A typical update scheme is given by
\begin{equation}\label{eq:step3}
	c^{(k+1)} = 
	\left\{ \begin{array}{lc} 
		\beta c^{(k)} & \mathrm{if} \;\; h(\bbW^{(k+1)}) > \gamma h(\bbW^{(k)}) \\
		c^{(k)} & \mathrm{otherwise},
	 \end{array} \right.
\end{equation}
where $0 <\gamma < 1$ and $\beta > 1$ are positive constants.
Intuitively, $c^{(k)}$ is increased only if the constraint violation is not decreased by a factor of $\gamma$.

When the sequence of iterations is completed, the estimated DAG structure is given by $\hbW = \bbW^{(\kappa_{max})}$.
The convexity of the augmented Lagrangian guarantees that $\bbW^{(k)}$ corresponds to the global minimum of \eqref{eq:step1} for every $k$.
Therefore,~\cite[Prop. 4.2.1]{bertsekas1997dynamic} guarantees that every limit point of the sequence $\bbW^{(k)}$ is a global minimum of the constrained problem in \eqref{eq:nonneg_dag_learning}.

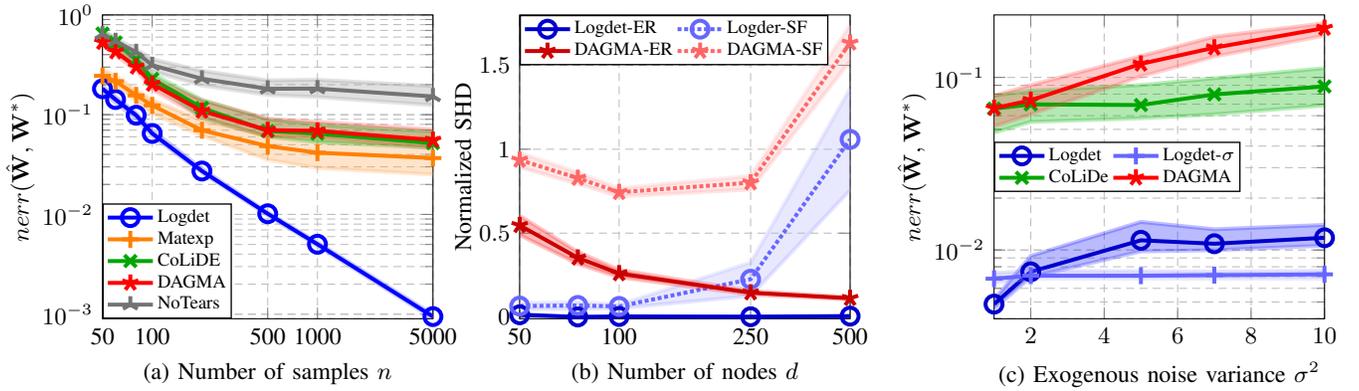
\begin{figure*}[!t]
	\centering
	\begin{subfigure}[t]{0.32\textwidth}
		\centering
		\begin{tikzpicture}[baseline,scale=1]

\pgfplotstableread{data/samples_err_med.csv}\errtable
\pgfplotstableread{data/samples_err_prctile25.csv}\prcttop
\pgfplotstableread{data/samples_err_prctile75.csv}\prctbot

\pgfmathsetmacro{\opacity}{0.3}
\pgfmathsetmacro{\contourop}{0.25}

\begin{loglogaxis}[
    xlabel={(a) Number of samples $n$},
    xmin=50,
    xmax=5000,
    xtick = {50, 100, 500, 1000, 5000},
    xticklabels = {50, 100, 500, 1000, 5000},
    ylabel={$nerr(\hbW, \bbW^*)$},
    ymin = 9e-4,
    ymax = 1,
    grid style=densely dashed,
    grid=both,
    legend style={
        at={(0, 0)},
        anchor=south west},
    legend columns=1,
    width=170,
    height=160,
    ]

    \addplot [blue!80!white, name path = logdet-bot, opacity=\contourop, forget plot] table [x=xaxis, y= MM-Logdet] \prctbot;
    \addplot [blue!70!white, name path = logdet-top, opacity=\contourop, forget plot] table [x=xaxis, y= MM-Logdet] \prcttop;
    \addplot[blue!70!white, fill opacity=\opacity, forget plot] fill between[of=logdet-bot and logdet-top];
    \addplot[blue, solid, mark=o] table [x=xaxis, y= MM-Logdet] {\errtable};

    \addplot [orange!80!white, name path = matexp-bot, opacity=\contourop, forget plot] table [x=xaxis, y= MM-Matexp] \prctbot;
    \addplot [orange!70!white, name path = matexp-top, opacity=\contourop, forget plot] table [x=xaxis, y= MM-Matexp] \prcttop;
    \addplot[orange!70!white, fill opacity=\opacity, forget plot] fill between[of=matexp-bot and matexp-top];
    \addplot[orange, solid, mark=+] table [x=xaxis, y= MM-Matexp] {\errtable};

    \addplot [green!70!black, name path = colide-bot, opacity=\contourop, forget plot] table [x=xaxis, y= CoLiDe-Fix] \prctbot;
    \addplot [green!70!black, name path = colide-top, opacity=\contourop, forget plot] table [x=xaxis, y= CoLiDe-Fix] \prcttop;
    \addplot[green!70!black, fill opacity=\opacity, forget plot] fill between[of=colide-bot and colide-top];
    \addplot[green!65!black, solid, mark=x] table [x=xaxis, y= CoLiDe-Fix] {\errtable};
    
    \addplot [red!80!white, name path = dagma-bot, opacity=\contourop, forget plot] table [x=xaxis, y=DAGMA] \prctbot;
    \addplot [red!70!white, name path = dagma-top, opacity=\contourop, forget plot] table [x=xaxis, y=DAGMA] \prcttop;
    \addplot[red!70!white, fill opacity=\opacity, forget plot] fill between[of=dagma-bot and dagma-top];
    \addplot[red, solid, mark=star] table [x=xaxis, y=DAGMA] {\errtable};

    \addplot [gray!90!white, name path = notears-bot, opacity=\contourop, forget plot] table [x=xaxis, y=NoTears] \prctbot;
    \addplot [gray!80!white, name path = notears-top, opacity=\contourop, forget plot] table [x=xaxis, y=NoTears] \prcttop;
    \addplot[gray!80!white, fill opacity=\opacity, forget plot] fill between[of=notears-bot and notears-top];
    \addplot[gray!90!black, solid, mark=Mercedes star] table [x=xaxis, y=NoTears] {\errtable};
    
    \legend{Logdet,  Matexp, CoLiDE, DAGMA, NoTears}
\end{loglogaxis}
\end{tikzpicture}
	\end{subfigure}
	\begin{subfigure}[t]{0.32\textwidth}
		\centering
		\begin{tikzpicture}[baseline,scale=1]

\pgfplotstableread{data/size_shd_mean.csv}\errtable
\pgfplotstableread{data/size_shd_std_up.csv}\prcttop
\pgfplotstableread{data/size_shd_std_down.csv}\prctbot

\pgfmathsetmacro{\opacity}{0.3}
\pgfmathsetmacro{\contourop}{0.25}

\begin{semilogxaxis}[
    xlabel={(b) Number of nodes $d$},
    xmin=50,
    xmax=500,
    xtick = {50, 100, 250, 500},
    xticklabels = {50, 100, 250, 500},
    ylabel={Normalized SHD},
    ymin = -0.01,
    ymax = 1.8,
    grid style=densely dashed,
    grid=both,
    legend style={
        at={(0, 1)},
        anchor=north west},
    legend columns=2,
    width=170,
    height=160,
    ]

    \addplot [blue, name path = logdet-er-bot, opacity=\contourop, forget plot] table [x=xaxis, y=  MM-Logdet-ER] \prctbot;
    \addplot [blue!90!white, name path = logdet-er-top, opacity=\contourop, forget plot] table [x=xaxis, y=  MM-Logdet-ER] \prcttop;
    \addplot[blue!90!white, fill opacity=\opacity, forget plot] fill between[of=logdet-er-bot and logdet-er-top];
    \addplot[blue!80!black, solid, mark=o] table [x=xaxis, y=  MM-Logdet-ER] {\errtable};

    \addplot [blue!40!white, name path = logdet-sf-bot, opacity=\contourop, forget plot] table [x=xaxis, y=  MM-Logdet-SF] \prctbot;
    \addplot [blue!30!white, name path = logdet-sf-top, opacity=\contourop, forget plot] table [x=xaxis, y=  MM-Logdet-SF] \prcttop;
    \addplot[blue!30!white, fill opacity=\opacity, forget plot] fill between[of=logdet-sf-bot and logdet-sf-top];
    \addplot[blue!60!white, densely dotted, mark=o] table [x=xaxis, y=  MM-Logdet-SF] {\errtable};

    \addplot [red, name path = dagma-er-bot, opacity=\contourop, forget plot] table [x=xaxis, y=DAGMA-ER] \prctbot;
    \addplot [red!90!white, name path = dagma-er-top, opacity=\contourop, forget plot] table [x=xaxis, y=DAGMA-ER] \prcttop;
    \addplot[red!90!white, fill opacity=\opacity, forget plot] fill between[of=dagma-er-bot and dagma-er-top];
    \addplot[red!80!black, solid, mark=star] table [x=xaxis, y=DAGMA-ER] {\errtable};

    \addplot [red!40!white, name path = dagma-sf-bot, opacity=\contourop, forget plot] table [x=xaxis, y=DAGMA-SF] \prctbot;
    \addplot [red!30!white, name path = dagma-sf-top, opacity=\contourop, forget plot] table [x=xaxis, y=DAGMA-SF] \prcttop;
    \addplot[red!30!white, fill opacity=\opacity, forget plot] fill between[of=dagma-sf-bot and dagma-sf-top];
    \addplot[red!60!white, densely dotted, mark=star] table [x=xaxis, y=DAGMA-SF] {\errtable};

    \legend{Logdet-ER, Logder-SF, DAGMA-ER, DAGMA-SF}
\end{semilogxaxis}
\end{tikzpicture}
	\end{subfigure}
	\begin{subfigure}[t]{0.32\textwidth}
		\centering
		\begin{tikzpicture}[baseline,scale=1]

\pgfplotstableread{data/vars_err_med.csv}\errtable
\pgfplotstableread{data/vars_err_prctile25.csv}\prcttop
\pgfplotstableread{data/vars_err_prctile75.csv}\prctbot

\pgfmathsetmacro{\opacity}{0.3}
\pgfmathsetmacro{\contourop}{0.25}

\begin{semilogyaxis}[
    xlabel={(c) Exogenous noise variance $\sigma^2$},
    xmin=1,
    xmax=10,
    ylabel={$nerr(\hbW, \bbW^*)$},
    ymin = 4e-3,
    ymax = .23,
    grid style=densely dashed,
    grid=both,
    legend style={
        at={(0, .5)},
        anchor=west},
    legend columns=2,
    width=170,
    height=160,
    ]

    \addplot [blue, name path = logdet-bot, opacity=\contourop, forget plot] table [x=xaxis, y=   MM-Logdet] \prctbot;
    \addplot [blue!90!white, name path = logdet-top, opacity=\contourop, forget plot] table [x=xaxis, y=   MM-Logdet] \prcttop;
    \addplot[blue!90!white, fill opacity=\opacity, forget plot] fill between[of=logdet-bot and logdet-top];
    \addplot[blue!80!black, solid, mark=o] table [x=xaxis, y=   MM-Logdet] {\errtable};

    \addplot [blue!40!white, name path = logdet-sigma-bot, opacity=\contourop, forget plot] table [x=xaxis, y=   MM-Logdet-Sigma] \prctbot;
    \addplot [blue!30!white, name path = logdet-sigma-top, opacity=\contourop, forget plot] table [x=xaxis, y=   MM-Logdet-Sigma] \prcttop;
    \addplot[blue!30!white, fill opacity=\opacity, forget plot] fill between[of=logdet-sigma-bot and logdet-sigma-top];
    \addplot[blue!60!white, solid, mark=+] table [x=xaxis, y=   MM-Logdet-Sigma] {\errtable};

    \addplot [green!70!black, name path = colide-bot, opacity=\contourop, forget plot] table [x=xaxis, y= CoLiDe-Fix] \prctbot;
    \addplot [green!70!black, name path = colide-top, opacity=\contourop, forget plot] table [x=xaxis, y= CoLiDe-Fix] \prcttop;
    \addplot[green!70!black, fill opacity=\opacity, forget plot] fill between[of=colide-bot and colide-top];
    \addplot[green!65!black, solid, mark=x] table [x=xaxis, y= CoLiDe-Fix] {\errtable};
    
    \addplot [red!80!white, name path = dagma-bot, opacity=\contourop, forget plot] table [x=xaxis, y=DAGMA] \prctbot;
    \addplot [red!70!white, name path = dagma-top, opacity=\contourop, forget plot] table [x=xaxis, y=DAGMA] \prcttop;
    \addplot[red!70!white, fill opacity=\opacity, forget plot] fill between[of=dagma-bot and dagma-top];
    \addplot[red, solid, mark=star] table [x=xaxis, y=DAGMA] {\errtable};
    
    \legend{Logdet, Logdet-$\sigma$, CoLiDe, DAGMA}
\end{semilogyaxis}
\end{tikzpicture}
	\end{subfigure}
		\vspace{-0.15cm}
	\caption{Evaluation of the proposed DAG structure learning method across different scenarios. a) reports the error of $\hbW$ as the number of samples increases.  b) presents the normalized SHD between $\hbW$ and the true DAG structure as the number of nodes increases. c) illustrates the error of $\hbW$ for different values of the variance of the exogenous input $\bbZ$.}
    \vspace{-.4cm}
    \label{fig:exps}
\end{figure*}

Finally, we demonstrate that the above algorithm can recover the true DAG structure $\bbW^*$ in the infinite sample regime, i.e., when the distribution of the random vector $\bbx$ is known.
To that end, replace the score function in \eqref{eq:nonneg_dag_learning} with
\begin{equation}\label{eq:loss_expectation}
    \bar{F}(\bbW, \bbx) = \mbE_\bbx \left[ \left\| \bbSigma_{\bbz}^{-\frac{1}{2}} \left( \bbI - \bbW^\top \right) \bbx \right\|_2^2 \right].
\end{equation}
Then, the next theorem guarantees the recovery of $\bbW^*$.

\begin{theorem}\label{th:recoverability}
Consider the score function $\bar{F}(\bbW, \bbx)$ in \eqref{eq:loss_expectation} and the convex acyclicity function $h_{ldet}(\bbW)$ from \eqref{eq:cvx_logdet}. Let $\bbx$ be a random vector following a linear SEM with non-negative DAG $\bbW^*\geq 0$ and exogenous input $\bbz$ with covariance $\bbSigma_\bbz$ known up to a scaling factor. Then, the estimate $\hbW$ from solving
    \begin{equation}\label{eq:expectation_dag_learning}
        \min_{\bbW} \; \bar{F}(\bbW, \bbx) \quad \mathrm{s.to} \quad \bbW \geq 0, \; h_{ldet}(\bbW) = 0,
    \end{equation}
    with the iterates \eqref{eq:step1}-\eqref{eq:step3}, satisfies
    $\hbW = \bbW^*$.
\end{theorem}
\begin{proof}
    Given the convexity of the optimization problem \eqref{eq:step1} and the updates in \eqref{eq:step2} and \eqref{eq:step3}, from~\cite[Prop. 4.2.1]{bertsekas1997dynamic} it follows that $\hbW$ is the global minimizer of \eqref{eq:nonneg_dag_learning}.
    Then, from \cref{prop:cvx_logdet}, we have that $h_{ldet}(\bbW) = 0$ if and only if $\bbW \in \mbD$.
    Consequently, minimizing \eqref{eq:expectation_dag_learning} is equivalent to solving
    \begin{equation}\label{eq:expectation_comb_dag_learning}
        \tbW \;=\; \arg\min_{\bbW} \; \bar{F}(\bbW, \bbx) \quad \mathrm{s.to} \quad \bbW \geq 0, \; \bbW \in \mbD,
    \end{equation}
    so $\hbW = \tbW$.
    
    Finally, from~\cite[Thm. 7]{loh2014high} it follows that the global minimizer of \eqref{eq:expectation_comb_dag_learning}, $\tbW$, corresponds to the true DAG structure, $\bbW^*$.
    Therefore, the proof is concluded since $\hbW = \tbW = \bbW^*$.
\end{proof}

\cref{th:recoverability} guarantees that our proposed method recovers the true DAG $\bbW^*$ when the distribution of $\bbx$ is known, which is tantamount to assuming access to infinite observations $n$.
While such an assumption is unlikely to be satisfied in practical settings, this preliminary recoverability guarantee brings to light the potential benefits of relying on convex acyclicity functions to learn the DAG structure.  
In the following section, we complete the assessment of our method by empirically evaluating its performance in the finite sample regime.

\section{Numerical experiments}\label{sec:exps}
We now evaluate the performance of the proposed method across different scenarios and compare it with state-of-the-art alternatives.
The code with the proposed method and all implementation details is publicly available on GitHub\footnote{\url{https://github.com/reysam93/cvx_dag_learning}}.

We measure the performance in terms of the normalized Frobenius error, calculated as
\begin{equation}
    nerr(\hbW, \bbW^*) = \| \bbW^* - \hbW \|_F^2 / \| \bbW^* \|_F^2,
\end{equation}
and the structural Hamming distance (SHD) normalized by the number of nodes, which counts the number of edges in $\hbW$ that need to be changed to match the support of $\bbW^*$.
As baselines, we consider the following relevant non-convex approaches: NO TEARS~\cite{zheng2018dags}, DAGMA~\cite{bello2022dagma}, and CoLiDE~\cite{saboksayr2023colide}.
In addition, we consider the convex constraints \eqref{eq:cvx_logdet} and \eqref{eq:cvx_matexp}, respectively denoted as ``Logdet'' and ``Matexp'', in \cref{fig:exps}.
Unless otherwise stated, we simulate Erd\H{o}s-Rényi (ER) graphs with $100$ nodes and average degree of 4, as well as $1000$ samples following a linear SEM with $\bbz$ sampled from a standard Gaussian distribution.
We report the median and the 25th and 75th percentiles of 100 independent realizations.

\vspace{2mm}
\noindent
\textbf{Test case 1.}
The first experiment examines the error $nerr(\hbW, \bbW^*)$ of different methods as the number of samples $n$ increases, as indicated on the x-axis.
From the results in \cref{fig:exps} (a), it is evident that leveraging a convex acyclicity constraint consistently leads to superior performance.
Moreover, while the error for the considered baselines saturates, the error associated with our convex constraint ``Logdet'' approaches 0 as the number of samples increases. This behavior aligns with \cref{th:recoverability}, showcasing the potential of our convex method for recovering the true DAG structure. Additionally, although both ``Matexp'' and ``Logdet'' use convex constraints, the latter outperforms the former.
This highlights the benefits of using the log determinant to ensure acyclicity, which coincides with the conclusions drawn for non-convex approaches~\cite{bello2022dagma}. 

\vspace{2mm}
\noindent
\textbf{Test case 2.} 
Next, \cref{fig:exps} (b) depicts the normalized SHD as the number of nodes $d$ increases, with the number of samples fixed at $n=1000$.
This experiment considers both ER graphs and scale-free (SF) graphs.
Our results demonstrate that ``Logdet'' consistently outperforms ``DAGMA''.
Since both methods utilize an acyclicity constraint based on the log determinant [cf. \eqref{eq:dagma} and \eqref{eq:cvx_logdet}], the difference in performance underscores the advantages of exploiting the non-negativity and convexity of $h_{ldet}(\bbW)$. 
Specifically, for ER graphs, our method achieves a normalized SHD of zero, showcasing that it accurately recovers the support of the true DAG even in the small-sample regime.
Regarding the SF graphs, the performance of the non-convex ``DAGMA'' method deteriorates significantly more than that of ``Logdet'' as the number of nodes increases.
Overall, the experiment also suggests that estimating non-negative SF DAGs is more challenging than estimating ER.

\vspace{2mm}
\noindent
\textbf{Test case 3.}
To conclude the numerical evaluation, we assume  the covariance of the exogenous input is given by $\bbSigma_\bbz = \sigma^2\bbI$, and evaluate the performance as $\sigma^2$ increases. We also include a variant of \eqref{eq:nonneg_dag_learning}, which leverages the knowledge of $\bbSigma_\bbz$, referred to as ``Logdet-$\sigma$''.
Consistent with previous results, \cref{fig:exps} (c) demonstrates that our proposed method based on the convex log determinant consistently outperforms the non-convex approaches.
Moreover, while the performance of ``DAGMA'' and ``Logdet'' deteriorates for larger values of $\sigma^2$, the error of ``CoLiDE'' and ``Logdet-$\sigma$'' remains stable, highlighting the advantages of leveraging information about $\bbSigma_\bbz$ or, if unavailable, estimating it as in ``CoLiDE''~\cite{saboksayr2023colide}.

\section{Conclusion}\label{S:conc}
This paper studied the prominent task of learning the structure of a DAG from a set of nodal observations using a linear SEM.
Leveraging the assumed non-negativity of $\bbW$, we framed DAG learning as a continuous optimization problem and introduced the first method that guarantees the estimated DAG corresponds to a global minimizer.
Specifically, we demonstrated that a convex constraint based on the log determinant ensures acyclicity when the graph does not contain negative edges.
The convexity of the acyclicity function enables us to formulate an abstract convex optimization problem, which we solve with an iterative algorithm based on the method of multipliers, recovering the global minimum.
After establishing preliminary recovery guarantees by demonstrating that our method recovers the ground truth DAG in the infinite sample size regime, we validated its performance through reproducible numerical experiments on synthetic data, where it outperformed state-of-the-art methods.

\newpage

\bibliographystyle{IEEEbib}
\bibliography{myIEEEabrv,biblio}

\end{document}